\documentclass[runningheads]{llncs}
\usepackage{graphicx}
\usepackage{amssymb}
\usepackage{bm}
\usepackage{mathtools}
\begin{document}
\title{Generalization Bounds for Set-to-Set Matching with Negative Sampling}
\author{Masanari Kimura\inst{1}\orcidID{0000-0002-9953-3469}}
\authorrunning{M. Kimura}
\institute{ZOZO Research, Tokyo, Japan\\ \email{masanari.kimura@zozo.com}}
\maketitle              % typeset the header of the contribution
\begin{abstract}
The problem of matching two sets of multiple elements, namely set-to-set matching, has received a great deal of attention in recent years.
In particular, it has been reported that good experimental results can be obtained by preparing a neural network as a matching function, especially in complex cases where, for example, each element of the set is an image.
However, theoretical analysis of set-to-set matching with such black-box functions is lacking.
This paper aims to perform a generalization error analysis in set-to-set matching to reveal the behavior of the model in that task.

\end{abstract}
\section{Introduction}
The problem of matching two sets of multiple elements, namely set-to-set matching, has received a great deal of attention in recent years~\cite{iwata2015unsupervised,kimura2021shift15m,lisanti2017group,xiao2018group}.
The problem is formalized as a task that, given two distinct sets, finds the goodness of match between them.
In particular, when the elements of the set are high-dimensional, neural networks are used as the matching function~\cite{saito2020exchangeable}.
Although these strategies have been reported to work well experimentally, there is a lack of research on their theoretical behavior.
A mathematical understanding of the behavior of the algorithm is an important issue since a lack of theoretical research hinders the improvement of existing algorithms for set-to-set matching.

We aim to perform a generalization error analysis of set-to-set matching algorithms in the context of statistical learning theory~\cite{vapnik1999overview,vapnik1999nature}.
In particular, existing deep learning-based set-to-set matching algorithms rely on negative sampling, a procedure in which negative examples are randomly generated while learning process~\cite{saito2020exchangeable}.
Therefore, we clarify the theoretical behavior of the set-to-set matching algorithm with negative sampling.

\section{Preliminaries}
Let $\bm{x}_n, \bm{y}_m \in \mathfrak{X} = \mathbb{R}^d$ be $d$-dimensional feature
vectors representing the features of each individual item.
Let $\mathcal{X} = \{\bm{x}_1,\dots,\bm{x}_N\}$ and $\mathcal{Y} = \{\bm{y}_1,\dots,\bm{y}_M\}$ be sets of these feature vectors, where $\mathcal{X},\mathcal{Y} \in 2^{\mathfrak{X}}$ and $N,M\in\mathbb{N}$ are sizes of the sets.
The function $f: 2^{\mathfrak{X}}\times 2^{\mathfrak{X}} \to \mathbb{R}$ calculates a matching score between the two sets $\mathcal{X}$ and $\mathcal{Y}$.
Guaranteeing the exchangeability of the set-to-set matching requires that the matching function $f(\mathcal{X}, \mathcal{Y})$ is symmetric and invariant under any permutation of items within each set as follows.
\begin{definition}[Permutation Invariance]
A set-input function f is said to be permutation invariant if
\begin{align}
    f(\mathcal{X}, \mathcal{Y}) = f(\pi_x\mathcal{X} , \pi_y\mathcal{Y})
\end{align}
for permutations $\pi_x$ on $\{1,\dots, N\}$ and $\pi_y$ on $\{1,\dots, M\}$.
\end{definition}
\begin{definition}[Permutation Equivariance]
A map $f: \mathfrak{X}^N \times \mathfrak{X}^M\to \mathfrak{X}^N$ is said to be permutation equivariant if
\begin{align}
    f(\pi_x\mathcal{X} , \pi_y\mathcal{Y}) = \pi_x f(\mathcal{X}, \mathcal{Y})
\end{align}
for permutations $\pi_x$ and $\pi_y$, where $\pi_x$ and $\pi_y$ are on $\{1,\dots,N\}$ and $\{1,\dots, M\}$, respectively.
Note that $f$ is permutation invariant for permutations within $\mathcal{Y}$.
\end{definition}
\begin{definition}[Symmetric Function]
A map $f: 2^\mathfrak{X}\times 2^{\mathfrak{X}} \to \mathbb{R}$ is said to be symmetric if
\begin{align}
    f(\mathcal{X}, \mathcal{Y}) = f(\mathcal{Y}, \mathcal{X}).
\end{align}
\end{definition}
\begin{definition}[Two-Set-Permutation Equivariance]
Given $\mathcal{X}^{(1)}\in \mathfrak{X}^N$ and $\mathcal{Z}^{(2)} \in \mathfrak{X}^M$, a map $f: \mathfrak{X}^\ast \times \mathfrak{X}^\ast \to \mathfrak{X}^\ast \times \mathfrak{X}^\ast$ is said to be two-set-permutation equivariant if
\begin{align}
    pf(\mathcal{Z}^{(1)}, \mathcal{Z}^{(2)}) = f(\mathcal{Z}^{(p(1))}, \mathcal{Z}^{(p(2))})
\end{align}
for any permutation operator $p$ exchanging the two sets, where $\mathfrak{X}^\ast = \cup^\infty_{n=0} \mathfrak{X}^n$ indicates a sequence of arbitrary length such as $\mathfrak{{X}^N}$ or $\mathfrak{X}^M$.
\end{definition}

We consider tasks where the matching function f is used per pair of sets~\cite{zhu2013point} to select a correct matching. 
Given candidate pairs of sets $(\mathcal{X}, \mathcal{Y}^{(k)})$, where $\mathcal{X}, \mathcal{Y}^{(k)} \in 2^{\mathfrak{X}}$ and $k\in \{1,\dots, K\}$, we choose $\mathcal{Y}^{(k^\ast)}$ as a correct one so that $f(\mathcal{X} , \mathcal{Y}^{(k^\ast)})$ achieves the maximum score from amongst the $K$ candidates.

\subsection{Set-to-set matching with negative sampling}
In real-world set-to-set matching problems, it is often the case that only positive example set pairs can be obtained.
Then, we consider training a model for set-to-set matching with negative sampling.
The learner is given positive examples $S^+ = \{(\mathcal{X}, \mathcal{Y})\}^{m^+}_{i=1}$.
Then, negative examples $S^- = \{(\mathcal{X}, \mathcal{Y})\}^{m^-}_{i=1}$ are generated by randomly combining set pairs from the given sets.
We assume that positive and negative examples are drawn according to the underlying distribution
$p^+$ and $p^-$, respectively.
Given training sample set $S = (S^+, S^-)$, the goal of set-to-set matching with negative sampling is to learn a real-valued score function $f:2^{\mathfrak{X}}\times 2^{\mathfrak{X}}\to\mathbb{R}$ that ranks future positive pair $(\mathcal{X}, \mathcal{Y})^+$ higher than negative pair $(\mathcal{X}, \mathcal{Y})^-$.
Let $\ell$ be the loss function, which is defined as
\begin{align}
    \ell(f, Z^+, Z^-) \coloneqq \varphi(f(Z^+) - f(Z^-)),
\end{align}
where $Z^+ = (\mathcal{X}, \mathcal{Y})^+$, $Z^- = (\mathcal{X}, \mathcal{Y})^-$ and $\varphi:\mathbb{R}\to\mathbb{R}^+$ is a convex function.
Typical choices of $\varphi$ include the logistic loss
\begin{align}
    \varphi(f(Z^+) - f(Z^-)) = \log\left\{1 + \exp(- (f(Z^+) - f(Z^-))\right\}.
\end{align}
\begin{definition}[Expected set-to-set matching loss]
Expected set-to-set matching loss $R(f)$ is defined as
\begin{align}
    R(f) &\coloneqq \mathbb{E}_{Z^+\sim p^+,Z^-\sim p^-}\left[\ell(f, Z^+, Z^-)\right].
\end{align}
\end{definition}
\begin{definition}[Empirical set-to-set matching loss]
Empirical set-to-set matching loss $\hat{R}(f;S)$ is defined as
\begin{align}
    \hat{R}(f;S) \coloneqq \frac{1}{m^+m^-}\sum^{m^+}_{i=1}\sum^{m^+ + m^-}_{j=m^+ + 1}\ell(f, Z^+, Z^-).
\end{align}
\end{definition}
Here, we assume that $\varphi$ has the Lipschitz property with respect to $\mathbb{R}$, i.e.,
\begin{align}
    \left|\varphi(a) - \varphi(b)\right| \leq L\cdot\left|a - b\right|,
\end{align}
where $a,b\in\mathbb{R}$ and $L>0$ is a Lipschitz constant.

\section{Margin bound for set-to-set matching}
Our first result is based on the Rademacher complexity.
\begin{definition}[Empirical Rademacher complexity]
Let $\mathcal{F}$ be a family of matching score functions.
Then, the empirical Rademacher complexity of $\mathcal{F}$ with respect to the sample $S$ is defined as
\begin{align}
    \hat{\mathcal{R}}_S(\mathcal{F}) \coloneqq \mathbb{E}_\sigma\left[\sup_{f\in\mathcal{F}}\frac{1}{m}\sum^m_{i=1}\sigma_i f(Z_i)\right].
\end{align}
\end{definition}
\begin{definition}[Rademacher complexity]
Let $p$ denote the distribution according to which samples are drawn.
For any integer $m\geq 1$, the Rademacher complexity of $\mathcal{F}$ is the expectation of the empirical Rademacher complexity over all samples of size $m$ drawn according to $p$:
\begin{align}
    \mathcal{R}_m(\mathcal{F}) \coloneqq \mathbb{E}_{S\sim p^m}\left[\hat{\mathcal{R}}(\mathcal{F})\right].
\end{align}
\end{definition}

Let $p_1$ the marginal distribution of the first element of the pairs, and by $p_2$ the marginal distribution with respect to the second element of the pairs.
Similarly, $S^1\sim p_1$ and $S^2\sim p_2$.
We denote by $\mathcal{R}^1_m$ the Rademacher complexity of $\mathcal{F}$ with respect to the marginal distribution $p_1$, that is $\mathcal{R}^1_m(\mathcal{F})=\mathbb{E}[\hat{\mathcal{R}}_{S^1}(\mathcal{F})]$, and similarly $\mathcal{R}^2_m(\mathcal{F})=\mathbb{E}[\hat{\mathcal{R}}_{S^2}(\mathcal{F})]$.

Here, we assume that the loss function is the following margin loss.
\begin{definition}
For any $\rho>0$, the $\rho$-margin loss is the function $\ell_\rho$ defined for all $z,z'\in\mathbb{R}$ by $\ell_\rho(z, z')=\phi(zz')$ with,
\begin{align}
    \phi_\rho(z) = \begin{cases}
    0 & (\rho \leq z) \\
    1 - z / \rho & (0 \leq z \leq \rho) \\
    1 & (z \leq 0).
    \end{cases}
\end{align}
\end{definition}

\begin{lemma}
\label{lem:margin_bound}
Let $Z\in\mathbb{R}$ be any input space, and $\mathcal{G}$ be a family of functions mapping from $Z$ to $[0,1]$.
Then, for any $\delta>0$, with probability at least $1-\delta$, each of the following holds for all $g\in\mathcal{G}$:
\begin{align}
    \mathbb{E}[g(z)] &\leq \frac{1}{m}\sum^m_{i=1}g(z_i) +2\mathcal{R}_m(\mathcal{G}) + \sqrt{\frac{\log\frac{1}{\delta}}{2m}}, \\
    \mathbb{E}[g(z)] &\leq \frac{1}{m}\sum^m_{i=1}g(z_i) + 2\hat{\mathcal{R}}_S(\mathcal{G}) + 3\sqrt{\frac{\log\frac{2}{\delta}}{2m}}.
\end{align}
\end{lemma}
\begin{proof}
Let $\psi(S) = \sup_{g\in\mathcal{G}}\mathbb{E}[g] - \frac{1}{m}\sum^m_{i=1}g(z_i)$.
Then, for two samples $S$ and $S'$, we have
\begin{align*}
\psi(S') -\psi(S) \leq \sup_{g\in\mathcal{G}}\frac{g(z_m) - g(z'_m)}{m} \leq \frac{1}{m}.
\end{align*}
where $z_m\in S$ and $z'_m \in S'$.
Then, by McDiarmid's inequality, for any $\delta > 0$, with probability at least $1-\delta/2$, the following holds.
\begin{align*}
    \psi(S) \leq \mathbb{E}_S[\psi(S)] + \sqrt{\frac{\log\frac{2}{\delta}}{2m}}.
\end{align*}
We next bound the expectation of the right-hand side as follows.
\begin{align*}
    \mathbb{E}_S[\psi(S)] &= \mathbb{E}_S\left[\sup_{g\in\mathcal{G}}\mathbb{E}[g] - \frac{1}{m}\sum^m_{i=1}g(z_i)\right] \\
    &\leq \mathbb{E}_{S,S'}\left[\sup_{g\in\mathcal{G}}\frac{1}{m}\sum^m_{i=1}(g(z'_i) - g(z_i))\right] \\
    &= \mathbb{E}_{\sigma, S, S'}\left[\sup_{g\in\mathcal{G}}\frac{1}{m}\sum^m_{i=1}\sigma_i(g(z'_i) - g(z_i))\right] \\
    &\leq \mathbb{E}_{\sigma, S'}\left[\sup_{g\in\mathcal{G}}\frac{1}{m}\sum^m_{i=1}\sigma_i g(z'_i)\right] + \mathbb{E}_{\sigma,S}\left[\sup_{g\in\mathcal{G}}\frac{1}{m}\sum^m_{i=1}-\sigma_i g(z_i)\right] \\
    &= 2\mathbb{E}_{\sigma, S}\left[\sup_{g\in\mathcal{G}}\frac{1}{m}\sum^m_{i=1}\sigma_i g(z_i)\right] = 2\mathcal{R}_m(\mathcal{G}).
\end{align*}
Here, using again McDiarmid's inequality, with probability at least $1-\delta/2$, the following holds.
\begin{align*}
    \mathcal{R}_m(\mathcal{G}) \leq \hat{\mathcal{R}}_S(\mathcal{G}) + \sqrt{\frac{\log\frac{2}{\delta}}{2m}}.
\end{align*}
Finally, we use the union bound which yields with probability at least $1-\delta$:
\begin{align}
    \phi(S) \leq 2\hat{\mathcal{R}}_S(\mathcal{G}) + 3\sqrt{\frac{\log\frac{2}{\delta}}{2m}}.
\end{align}
\end{proof}

\begin{theorem}[Margin bound for set-to-set matching]
Let $\mathcal{F}$ be a set of matching score functions.
Fix $\rho>0$. Then, for any $\delta>0$, with probability at least $1-\delta$ over the choice of a sample $S$ of size $m$, each of the following holds for all $f\in\mathcal{F}$:
\begin{align}
    R(f) &\leq \hat{R}_\rho(f) + \frac{2}{\rho}\left(\mathcal{R}^1_m(\mathcal{F}) + \mathcal{R}^2_m(\mathcal{F})\right) + \sqrt{\frac{\log\frac{1}{\delta}}{2m}}, \\
    R(f) &\leq \hat{R}_\rho(f) + \frac{2}{\rho}\left(\hat{\mathcal{R}}_{S^1}(\mathcal{F}) + \hat{\mathcal{R}}_{S^2}(\mathcal{F})\right) + 3\sqrt{\frac{\log\frac{2}{\delta}}{2m}}.
\end{align}
\end{theorem}
\begin{proof}
Let $\tilde{\mathcal{F}}$ be the family of functions mapping $(\mathfrak{X}\times\mathfrak{X})\times\{-1, +1\}$ to $\mathbb{R}$ defined by $\tilde{\mathcal{F}}=\{z = (Z', Z), a)\mapsto a[f(Z') - f(Z)]\ |\ f\in\mathcal{F}\}$, where $a \in \{0,1\}$.
Consider the family of functions $\tilde{\mathcal{F}}=\{\phi_\rho \circ g\ |\ f\in\tilde{\mathcal{F}}\}$ derived from $\tilde{\mathcal{F}}$ which are taking values in $[0,1]$. By Lemma~\ref{lem:margin_bound}, for any $\delta>0$ with probability at least $1-\delta$, for all $f\in\mathcal{F}$,
\begin{align}
    \mathbb{E}\left[\phi_\rho(a[f(Z^+) - f(Z^-)])\right] \leq \hat{R}_\rho(f) + 2\mathcal{R}_m(\phi_\rho \circ \tilde{\mathcal{F}}) + \sqrt{\frac{\log\frac{1}{\delta}}{2m}}.
\end{align}
Since $1_{u\leq 0}\leq \phi_\rho(u)$ for all $u\in\mathbb{R}$, the generalization error $R(f)$ is a lower bound on left-hand side, $R(f)=\mathbb{E}[1_{a[f(Z') - f(Z)]\leq 0}]\leq \mathbb{E}[\phi_\rho(a[f(Z') - f(Z)])]$, and we can write
\begin{align}
    R(f) \leq \hat{R}_\rho(f) + 2\mathcal{R}_m(\phi_\rho \circ \tilde{\mathcal{F}}) + \sqrt{\frac{\log\frac{1}{\delta}}{2m}}.
\end{align}
Here, we can show that $\mathcal{R}_m(\phi_\rho \circ \tilde{\mathcal{F}})\leq \frac{1}{\rho}\mathcal{R}_m(\tilde{\mathcal{F}})$ using the $(1/\rho)$-Lipschitzness of $\phi_\rho$.
Then, $\mathcal{R}_m(\tilde{\mathcal{F}})$ can be upper bounded as follows:
\begin{align*}
    \mathcal{R}_m(\tilde{\mathcal{F}}) &= \frac{1}{m}\mathbb{E}_{S,\sigma}\left[\sup_{f\in\mathcal{F}}\sum^m_{i=1}\sigma_i a_i(f(Z'_i) - f(Z_i))\right] \\
    &= \frac{1}{m}\mathbb{E}_{S,\sigma}\left[\sup_{f\in\mathcal{F}}\sum^m_{i=1}\sigma_i(f(Z'_i) - f(Z_i))\right] \\
    &\leq \frac{1}{m}\mathbb{E}_{S,\sigma}\left[\sup_{f\in\mathcal{F}}\sum^m_{i=1}\sigma_i f(Z'_i) + \sup_{f\in\mathcal{F}}\sum^m_{i=1}\sigma_i f(Z_i)\right] \\
    &= \mathbb{E}_S\left[\mathcal{R}_{S^2}(\mathcal{F}) + \mathcal{R}_{S^1}(\mathcal{F})\right] = \mathcal{R}^{p_2}_m(\mathcal{F}) + \mathcal{R}^{p_1}_m(\mathcal{F}).
\end{align*}
\end{proof}

\section{RKHS bound for set-to-set matching}
In this section, we consider more precise bounds that depend on the size of the negative sample produced by negative sampling.
Let $S = ((\mathcal{X}_1,\mathcal{Y}_1),\dots,(\mathcal{X}_m,\mathcal{Y}_m))\in(\mathfrak{X}\times\mathfrak{X})^m$ be a finite sample sequence, and $m^+$ be the positive sample size.
If the positive proportion $\frac{m^+}{m} = \alpha$, then sample sequence $S$ also can be denoted by $S_\alpha$.

\begin{figure}[t]
    \centering
    \includegraphics[width=0.95\linewidth]{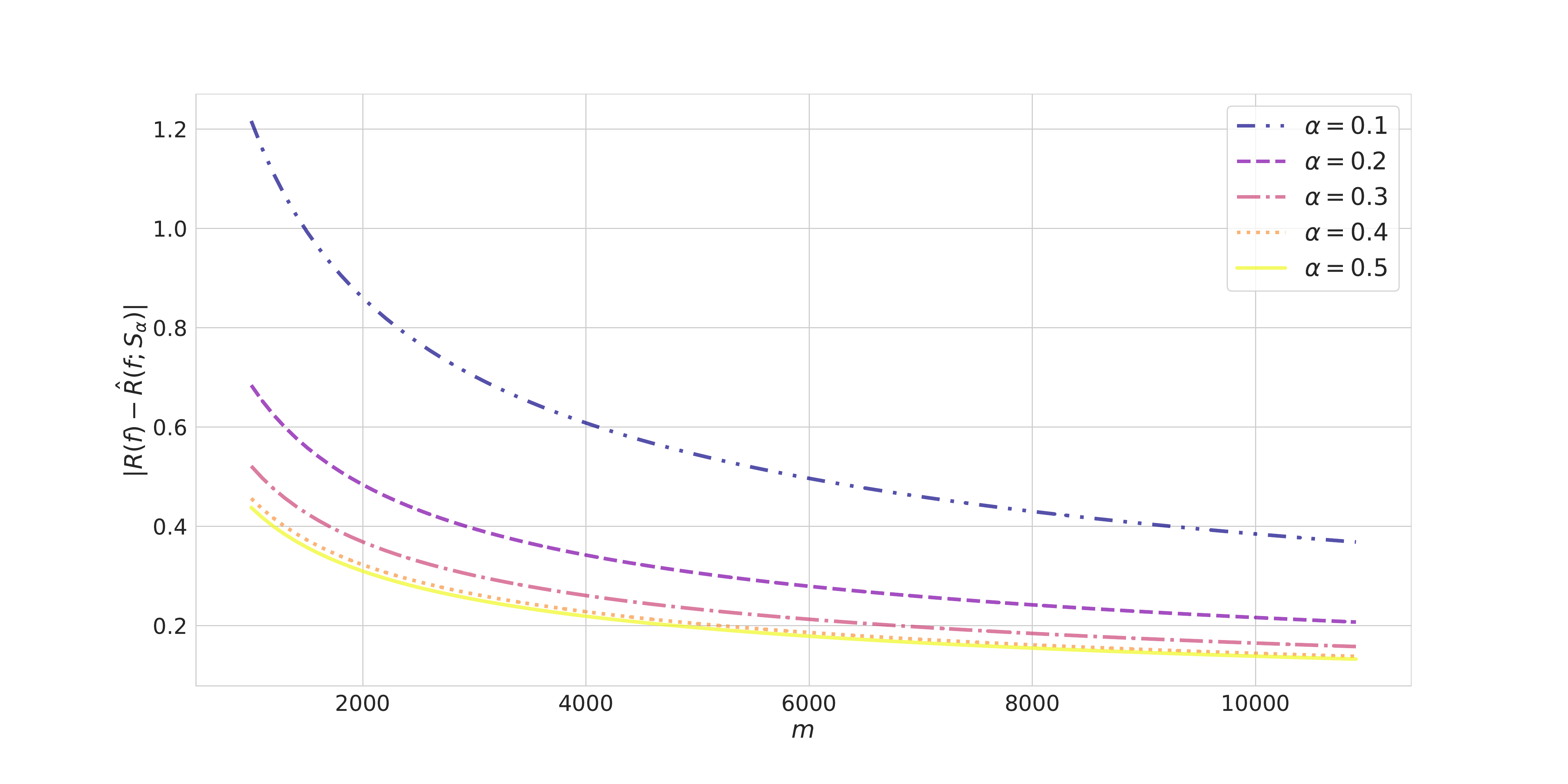}
    \caption{RKHS bound w.r.t. sample size $m$ and positive ratio $\alpha$.}
    \label{fig:rkhs_bound}
\end{figure}

Let $\mathfrak{R}_K$ be the reproducing kernel Hilbert space (RKHS) associated with the kernel $K$, and $\mathcal{F}_r$ is defined as
\begin{align}
    \mathcal{F}_r = \{f\in\mathfrak{R}_K\ |\ \|f\|_K\leq r\}
\end{align}
for $r>0$.
\begin{theorem}[RKHS bound for set-to-set matching]
\label{thm:rkhs_bound}
Suppose $S_\alpha$ to be any sample sequence of size $m$.
Then, for any $\epsilon > 0$ and $f\in\mathcal{F}_r$,
\begin{align}
    \mathbb{P}_{S_\alpha}\left[|\hat{R}(f; S_\alpha) - R(f)| \geq \epsilon\right] \leq 2\exp\left\{\frac{\alpha^2(1-\alpha)^2m\epsilon^2}{2L^2\kappa^2r^2}\right\},
\end{align}
where $\kappa \coloneqq \sup_{x}\sqrt{K(x, x)}$.
\end{theorem}
\begin{proof}
Denote $S = (S^+, S^-) = \{Z_1,\dots,Z_m\}$ and
\begin{align}
    z_i \coloneqq z(Z_i) \coloneqq = \begin{cases}
    +1 & (Z_i \in S^+), \\
    -1 & (Z_i \in S^-).
    \end{cases}
\end{align}
First, for each $1\leq k \leq m^+$ such that $z_i = +1$, let $(Z_k, +1)$ be replaced by $(Z'_k, +1)\in (\mathfrak{X}\times\mathfrak{X})\times \{-1, +1\}$, and we denote by $S^k$ as this sample.
Then,
\begin{align*}
    |\hat{R}(f;S) - \hat{R}(f;S^k)| &\leq \frac{1}{m^+m^-}\sum^{m^+ + m^-}_{j=m^+ + 1}\left|\varphi(f(Z_k) - f(Z_j)) - \varphi(f(Z'_k) - f(Z_j))\right| \\
    &\leq \frac{1}{m^+ m^-}\sum^{m^+ + m^-}_{j=m^+ + 1}L\cdot\left|f(Z_k) - f(Z_j) - f(Z'_k) + f(Z_j)\right| \\
    &= \frac{1}{m^+ m^-}\cdot m^- \cdot L\cdot \left|f(Z_k) - f(Z'_k)\right| \leq \frac{2L}{m^+}\|f\|_\infty.
\end{align*}
Next, for each $m^+ + 1 \leq k \leq m$ such that $z_i = -1$, let $(Z_l, -1)$ be replaced by $(Z'_k, -1)\in (\mathfrak{X}\times\mathfrak{X})\times \{-1, +1\}$ and we denote by $\bar{S}^k$ as this sample.
Similarly, we have
\begin{align*}
    |\hat{R}(f;S) - \hat{R}(f;\bar{S}^k)| &\leq\frac{2L}{m^+}\|f\|_\infty.
\end{align*}
Finally, for each $1\leq k \leq m^+$ such that $z_i = +1$, let $(Z_k, +1)$ be replaced by $(Z'_k, -1)\in (\mathfrak{X}\times\mathfrak{X})\times \{-1, +1\}$, and we denote by $\tilde{S}^k = \bar{S}^k\cup\{(Z_{m+1}, -1)\}$ as this sample.
Then, we have
\begin{align*}
    |\hat{R}(f;S) - \hat{R}(f;\tilde{S}^k)| &\leq \Gamma_1 + \Gamma_2,
\end{align*}
where $\Gamma_1 = |\hat{R}(f;S) - \hat{R}(f;S\cup\{Z_{m+1}, -1\})|$ and $\Gamma_2 = |\hat{R}(f;S\cup\{Z_{m+1}, -1\}) - \hat{R}(f;\tilde{S}^k)|$.
Since $\Gamma_1 \leq \frac{2L}{m^- + 1}\|f\|_\infty$ and $\Gamma_2 \leq \frac{2L}{m^+}\|f\|_\infty$, we have
\begin{align}
    |\hat{R}(f;S) - \hat{R}(f;\tilde{S}^k)| &\leq 2L\left(\frac{1}{m^+} + \frac{1}{m^- + 1}\right)\|f\|_\infty.
\end{align}
Combining them and applying McDiarmid’s inequality, we have the proof.
\end{proof}

\begin{remark}
Given $m,\epsilon,L$, we can find that the tight bound can be achieved when $\alpha = \frac{1}{2}$.
This means that it is desirable the number of positive samples be equal to the number of negative samples (See Figure~\ref{fig:rkhs_bound}).
\end{remark}
\begin{remark}
\label{remark:rkhs_bound}
For any $\delta>0$, with probability at least $1-\delta$, we have
\begin{align}
    \left|\hat{R}(f; S_\alpha) - R(f)\right| \leq \frac{L\kappa r}{\alpha(1-\alpha)}\sqrt{\frac{2\log\frac{2}{\delta}}{m}}.
\end{align}
\end{remark}
\begin{remark}
For Remark~\ref{remark:rkhs_bound}, Let $m = m^+ + m^-$ and fix $m^+\in\mathbb{N}$.
Then, we have the optimal negative sample size as $(1-\alpha) = 2/3$.
\end{remark}

\section{Conclusion and Discussion}
In this paper, we performed a generalization error analysis in set-to-set matching to reveal the behavior of the model in that task.
Our analysis reveals what the convergence rate of algorithms in set matching depend on the size of negative sample.
Future studies may include the following:
\begin{itemize}
    \item Derivation of tighter bounds. There are many types of mathematical tools for generalization error analysis of machine learning algorithms, and it is known that the tightness of the bounds depends on which one is used. For tighter bounds, it is useful to use mathematical tools not addressed in this paper~\cite{bartlett2005local,mcallester1999some,mcallester1999pac,millar1983minimax,duchi2013local}.
    \item Induction of novel set matching algorithms. It is expected to derive a novel algorithm based on the discussion of generalized error analysis.
    \item The effect of data augmentation for generalization error of set-to-set matching. Many data augmentation methods have been proposed to stabilize neural network learning, and theoretical analysis when these are used would be useful~\cite{kimura2021mixup,kimura2021understanding,van2001art,zhong2020random,shorten2019survey}.
\end{itemize}
%\bibliographystyle{splncs04}
%\bibliography{references}

\end{document}